\documentclass{article} % For LaTeX2e
\usepackage{colm2024_conference}

\usepackage{microtype}
\usepackage{hyperref}
\usepackage{url}
\usepackage{booktabs}
\definecolor{darkblue}{rgb}{0, 0, 0.5}
\hypersetup{colorlinks=true, citecolor=darkblue, linkcolor=darkblue, urlcolor=darkblue}
\usepackage{tabularx}
\usepackage{multirow}
\usepackage{hyperref}
\usepackage{url}
\usepackage{float}

\usepackage{amssymb}
\usepackage{amsthm}
\usepackage{caption}
\usepackage{subcaption}
\usepackage{graphicx}
\usepackage[symbol]{footmisc}
\usepackage{cancel}
\hypersetup{colorlinks,citecolor=blue,linkcolor=red,urlcolor=blue}
\usepackage{mathtools}
\usepackage{tabularx}
\usepackage{booktabs} 

\usepackage{arydshln}

\usepackage{xcolor} % Add this to your preamble

\definecolor{UserPromptColor}{RGB}{0, 102, 204} % Darker blue
\definecolor{ModelResponseColor}{RGB}{204, 102, 0} % Darker orange

\newcommand{\bx}{\boldsymbol{x}}

\newcommand{\bz}{\boldsymbol{z}}

\newcommand{\bw}{\boldsymbol{\mathrm{w}}}

\usepackage{color, colortbl}
\definecolor{LightCyan}{rgb}{0.88,1,1}
\definecolor{LightYellow}{rgb}{1,1,0.88}
\definecolor{DarkYellow}{rgb}{0.8352941176470589,0.7137254901960784,0.0392156862745098}
\definecolor{Yellow}{rgb}{1,1,0.0784313725490196}
\definecolor{mydarkgreen}{RGB}{69, 153, 44}

\newcommand{\sign}{\mathrm{sign}}

\definecolor{green}{RGB}{0, 128, 0}
\definecolor{red}{RGB}{245, 60, 60}
\newcommand{\colorcellauroc}[1]{%
  \ifdim #1 pt < 0.5pt
    \cellcolor{red!100}
  \else
    \ifdim #1 pt < 0.51pt
      \cellcolor{red!100}
    \else
      \ifdim #1 pt < 0.52pt
        \cellcolor{red!20}
      \else
        \ifdim #1 pt < 0.55pt
          \cellcolor{red!10}
        \else
          \ifdim #1 pt < 0.6pt
            \cellcolor{green!10}
          \else
            \ifdim #1 pt < 0.7pt
              \cellcolor{green!20}
            \else
              \ifdim #1 pt < 0.8pt
                \cellcolor{green!40}
              \else
                \ifdim #1 pt < 0.9pt
                  \cellcolor{green!100}
                \else
                  \cellcolor{green!100}
                \fi
              \fi
            \fi
          \fi
        \fi
      \fi
    \fi
  \fi
  #1
}

\newcommand{\colorcellece}[1]{%
  \ifdim #1 pt < 0.02pt
    \cellcolor{green!45}
  \else
    \ifdim #1 pt < 0.05pt
      \cellcolor{green!40}
    \else
      \ifdim #1 pt < 0.07pt
        \cellcolor{green!35}
      \else
        \ifdim #1 pt < 0.10pt
          \cellcolor{green!30}
        \else
          \ifdim #1 pt < 0.15pt
            \cellcolor{green!25}
          \else
            \ifdim #1 pt < 0.20pt
              \cellcolor{green!20}
            \else
              \ifdim #1 pt < 0.25pt
                \cellcolor{red!50}
              \else
                \ifdim #1 pt < 0.30pt
                  \cellcolor{red!100}
                \else
                  \cellcolor{red!100}
                \fi
              \fi
            \fi
          \fi
        \fi
      \fi
    \fi
  \fi
  #1
}

\newcommand{\colorcellecetwo}[1]{%
  \ifdim #1 pt < 0.035pt
    \cellcolor{green!45}
  \else
    \ifdim #1 pt < 0.04pt
      \cellcolor{green!40}
    \else
      \ifdim #1 pt < 0.045pt
        \cellcolor{green!35}
      \else
        \ifdim #1 pt < 0.05pt
          \cellcolor{green!30}
        \else
          \ifdim #1 pt < 0.055pt
            \cellcolor{green!25}
          \else
            \ifdim #1 pt < 0.06pt
              \cellcolor{green!20}
            \else
              \ifdim #1 pt < 0.065pt
                \cellcolor{red!50}
              \else
                \ifdim #1 pt < 0.07pt
                  \cellcolor{red!100}
                \else
                  \cellcolor{red!100}
                \fi
              \fi
            \fi
          \fi
        \fi
      \fi
    \fi
  \fi
  #1
}

\newcommand{\colorcellecethree}[1]{%
  \ifdim #1 pt < 0.110pt
    \cellcolor{green!45}
  \else
    \ifdim #1 pt < 0.120pt
      \cellcolor{green!40}
    \else
      \ifdim #1 pt < 0.125pt
        \cellcolor{green!35}
      \else
        \ifdim #1 pt < 0.130pt
          \cellcolor{green!30}
        \else
          \ifdim #1 pt < 0.135pt
            \cellcolor{green!25}
          \else
            \ifdim #1 pt < 0.140pt
              \cellcolor{green!20}
            \else
              \ifdim #1 pt < 0.145pt
                \cellcolor{red!50}
              \else
                \ifdim #1 pt < 0.150pt
                  \cellcolor{red!100}
                \else
                  \cellcolor{red!100}
                \fi
              \fi
            \fi
          \fi
        \fi
      \fi
    \fi
  \fi
  #1
}

\newcommand{\colorcellecefour}[1]{%
  \ifdim #1 pt < 0.142pt
    \cellcolor{green!45}
  \else
    \ifdim #1 pt < 0.146pt
      \cellcolor{green!40}
    \else
      \ifdim #1 pt < 0.148pt
        \cellcolor{green!35}
      \else
        \ifdim #1 pt < 0.150pt
          \cellcolor{green!30}
        \else
          \ifdim #1 pt < 0.152pt
            \cellcolor{green!25}
          \else
            \ifdim #1 pt < 0.154pt
              \cellcolor{green!20}
            \else
              \ifdim #1 pt < 0.156pt
                \cellcolor{red!50}
              \else
                \ifdim #1 pt < 0.158pt
                  \cellcolor{red!100}
                \else
                  \cellcolor{red!100}
                \fi
              \fi
            \fi
          \fi
        \fi
      \fi
    \fi
  \fi
  #1
}

\newcommand{\colorcellecefive}[1]{%
  \ifdim #1 pt < 0.217pt
    \cellcolor{green!45}
  \else
    \ifdim #1 pt < 0.228pt
      \cellcolor{green!40}
    \else
      \ifdim #1 pt < 0.233pt
        \cellcolor{green!35}
      \else
        \ifdim #1 pt < 0.238pt
          \cellcolor{green!30}
        \else
          \ifdim #1 pt < 0.243pt
            \cellcolor{green!25}
          \else
            \ifdim #1 pt < 0.248pt
              \cellcolor{green!20}
            \else
              \ifdim #1 pt < 0.253pt
                \cellcolor{red!50}
              \else
                \ifdim #1 pt < 0.258pt
                  \cellcolor{red!100}
                \else
                  \cellcolor{red!100}
                \fi
              \fi
            \fi
          \fi
        \fi
      \fi
    \fi
  \fi
  #1
}

\newcommand{\colorcellecesix}[1]{%
  \ifdim #1 pt < 0.106pt
    \cellcolor{green!45}
  \else
    \ifdim #1 pt < 0.107pt
      \cellcolor{green!40}
    \else
      \ifdim #1 pt < 0.109pt
        \cellcolor{green!35}
      \else
        \ifdim #1 pt < 0.111pt
          \cellcolor{green!30}
        \else
          \ifdim #1 pt < 0.113pt
            \cellcolor{green!25}
          \else
            \ifdim #1 pt < 0.115pt
              \cellcolor{green!20}
            \else
              \ifdim #1 pt < 0.117pt
                \cellcolor{red!50}
              \else
                \ifdim #1 pt < 0.119pt
                  \cellcolor{red!100}
                \else
                  \cellcolor{red!100}
                \fi
              \fi
            \fi
          \fi
        \fi
      \fi
    \fi
  \fi
  #1
}

\newcommand{\colorcellbrierone}[1]{%
  \ifdim #1 pt < 0.4206pt
    \cellcolor{green!45}
  \else
    \ifdim #1 pt < 0.430pt
      \cellcolor{green!40}
    \else
      \ifdim #1 pt < 0.450pt
        \cellcolor{green!35}
      \else
        \ifdim #1 pt < 0.470pt
          \cellcolor{green!30}
        \else
          \ifdim #1 pt < 0.490pt
            \cellcolor{green!25}
          \else
            \ifdim #1 pt < 0.540pt
              \cellcolor{green!20}
            \else
              \ifdim #1 pt < 0.545pt
                \cellcolor{red!50}
              \else
                \ifdim #1 pt < 0.550pt
                  \cellcolor{red!100}
                \else
                  \cellcolor{red!100}
                \fi
              \fi
            \fi
          \fi
        \fi
      \fi
    \fi
  \fi
  #1
}

\newcommand{\colorcellbriertwo}[1]{%
  \ifdim #1 pt < 0.115pt
    \cellcolor{green!45}
  \else
    \ifdim #1 pt < 0.118pt
      \cellcolor{green!40}
    \else
      \ifdim #1 pt < 0.120pt
        \cellcolor{green!35}
      \else
        \ifdim #1 pt < 0.122pt
          \cellcolor{green!30}
        \else
          \ifdim #1 pt < 0.124pt
            \cellcolor{green!25}
          \else
            \ifdim #1 pt < 0.126pt
              \cellcolor{green!20}
            \else
              \ifdim #1 pt < 0.129pt
                \cellcolor{green!10}
              \else
                \ifdim #1 pt < 0.130pt
                  \cellcolor{red!100}
                \else
                  \cellcolor{red!100}
                \fi
              \fi
            \fi
          \fi
        \fi
      \fi
    \fi
  \fi
  #1
}

\newcommand{\colorcellbrierthree}[1]{%
  \ifdim #1 pt < 0.316pt
    \cellcolor{green!45}
  \else
    \ifdim #1 pt < 0.3195pt
      \cellcolor{green!40}
    \else
      \ifdim #1 pt < 0.3200pt
        \cellcolor{green!35}
      \else
        \ifdim #1 pt < 0.321pt
          \cellcolor{green!30}
        \else
          \ifdim #1 pt < 0.322pt
            \cellcolor{green!25}
          \else
            \ifdim #1 pt < 0.3230pt
              \cellcolor{green!20}
            \else
              \ifdim #1 pt < 0.324pt
                \cellcolor{red!50}
              \else
                \ifdim #1 pt < 0.325pt
                  \cellcolor{red!100}
                \else
                  \cellcolor{red!100}
                \fi
              \fi
            \fi
          \fi
        \fi
      \fi
    \fi
  \fi
  #1
}

\newcommand{\colorcellbrierfour}[1]{%
  \ifdim #1 pt < 0.477pt
    \cellcolor{green!45}
  \else
    \ifdim #1 pt < 0.485pt
      \cellcolor{green!40}
    \else
      \ifdim #1 pt < 0.490pt
        \cellcolor{green!35}
      \else
        \ifdim #1 pt < 0.495pt
          \cellcolor{green!30}
        \else
          \ifdim #1 pt < 0.500pt
            \cellcolor{green!25}
          \else
            \ifdim #1 pt < 0.505pt
              \cellcolor{green!20}
            \else
              \ifdim #1 pt < 0.510pt
                \cellcolor{red!50}
              \else
                \ifdim #1 pt < 0.514pt
                  \cellcolor{red!100}
                \else
                  \cellcolor{red!100}
                \fi
              \fi
            \fi
          \fi
        \fi
      \fi
    \fi
  \fi
  #1
}

\theoremstyle{plain}
\newtheorem{theorem}{Theorem}[section]
\newtheorem{proposition}[theorem]{Proposition}

\theoremstyle{definition}

\theoremstyle{remark}

\title{Just rephrase it! Uncertainty estimation in closed-source language models via multiple rephrased queries}

\author{Adam Yang \\
University of Bristol\\
Bristol, United Kingdom\\
\texttt{adam.yang@bristol.ac.uk} \\
\And
Chen Chen \\
Nanyang University \\
Singapore\\
\texttt{CHEN1436@e.ntu.edu.sg} \\
\AND
Konstantinos Pitas \\
INRIA Grenoble Rhône-Alpes \\
Grenoble, France \\
\texttt{kostasp210@gmail.com}
}

\colmfinalcopy % Uncomment for camera-ready version, but NOT for submission.
\begin{document}

\maketitle

\begin{abstract}
State-of-the-art large language models are sometimes distributed as open-source software but are also increasingly provided as a closed-source service. These closed-source large-language models typically see the widest usage by the public, however, they often do not provide an estimate of their uncertainty when responding to queries. As even the best models are prone to ``hallucinating" false information with high confidence, a lack of a reliable estimate of uncertainty limits the applicability of these models in critical settings. We explore estimating the uncertainty of closed-source LLMs via multiple rephrasings of an original base query. Specifically, we ask the model, multiple rephrased questions, and use the similarity of the answers as an estimate of uncertainty. We diverge from previous work in i) providing rules for rephrasing that are simple to memorize and use in practice ii) proposing a theoretical framework for why multiple rephrased queries obtain calibrated uncertainty estimates. Our method demonstrates significant improvements in the calibration of uncertainty estimates compared to the baseline and provides intuition as to how query strategies should be designed for optimal test calibration.
\end{abstract}

\section{Introduction}

Since the introduction of ChatGPT \citep{brown2020language}, closed-source Large Language Models have seen incredibly rapid adoption by the general public, resulting in great productivity gains \citep{eloundou2023gpts}. At the same time, closed-source LLMs are prone to generating highly convincing but false information, a problem known as "hallucinating" \citep{huang2023survey,ji2023survey}. They are furthermore known to state this false information with high confidence \citep{kadavath2022language}. This combination presents a conundrum to users. Specifically, while on average LLM-generated text is useful, the unreliability of any individual LLM-generated text and the lack of an effective mechanism to separate reliable and unreliable generations necessitates that LLM answers be inspected and vetted by the user, especially for critical applications. This significantly slows the LLM usage pipeline. Furthermore, the typical user has limited access to the model (specifically he can only query the LLM with textual prompts), and thus standard approaches for uncertainty estimation \citep{guo2017calibration,arbel2023primer} in deep neural networks cannot be applied, as they typically require access to the deep neural networks logits. 

It is folk wisdom that one approach for estimating LLM uncertainty, even with such limited access to the model, is to query it multiple times \citep{wang2022self,xiong2023can}. This approach is based on the premise that LLM-generated text is frequently stochastic by design, as the next generated token is chosen through nucleus sampling \citep{holtzman2019curious} or top-k decoding \citep{fan2018hierarchical,radford2019language}. \cite{wang2022self} and \cite{xiong2023can} proposed to use the consistency of multiple answers as an estimate of uncertainty. \cite{xiong2023can} furthermore proposed to add "noise" to the base query at each repetition, through misleading hints. While adding noise at each query repetition has been shown to improve over using the internal stochasticity of the LLM, we believe that there is considerable room for improvement. Specifically:

\begin{itemize}
\item The current SOTA hint-based approach to submitting multiple noisy queries \citep{xiong2023can} is cumbersome for end users, as it requires memorization of the hint patterns. This in turn might significantly limit adoption.
\item A theoretical understanding of why multiple queries work in the top-1 decoding settings is currently lacking. Specifically, a clear understanding of which "noising" methods work and why would help the community design better noising rules. 
\item Furthermore, a more detailed understanding of when and why adding noise to queries helps in the top-k decoding setting, (which by itself results in multiple answers) would help avoid "noising" queries when this is unnecessary.
\end{itemize}

\begin{figure*}[t!]
    \centering

    \begin{subfigure}{\textwidth}
        \centering
        \includegraphics[width=0.9\linewidth]{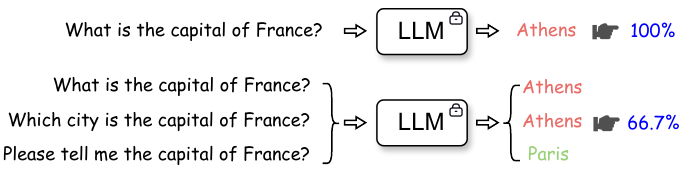}
        %\caption{CIFAR-100, ResNet20}
        \label{fig:subfig1}
    \end{subfigure}

    \caption{\textbf{Multiple rephrased queries for uncertainty estimation.} Top row: Querying a closed-source LLM only once with a base query may yield an incorrect top-1 prediction. In the absence of additional information, the naive baseline is to assign $100\%$ confidence to this singular prediction. Bottom row: Querying the model multiple times with rephrased versions of the base query produces the  $\{\mathrm{Athens}\}$ class twice and the $\{\mathrm{Paris}\}$ class once. This is roughly equivalent to $66.6\%$ confidence. This observation should serve as an alert to a potential error, even when the true label is unknown.}
    \label{fig:main}
\end{figure*}

In this work, we delve deeply in, refine, and theoretically analyze multiple queries for uncertainty estimation. Given a base query, we restrict ourselves to submitting rephrased versions of the base query to an LLM, checking the consistency of the answers, and using the result as an estimate of uncertainty. Concretely our contributions are the following:
\begin{itemize}
    \item  We test four simple strategies for creating multiple rephrased queries, and find that in the top-1 decoding setting, two of them, substituting words with their synonyms and making the base query more verbose, result in significant calibration gains over the naive baseline of trusting every LLM answer. These two strategies have the advantage of requiring only basic language and arithmetic skills by the end user, and practically no memorization apart from the rephrasing rule. 
    \item  We propose a theoretical model for multiple rephrased queries on a simplified top-1 decoding setting. Given multiple rephrased queries, our analysis shows that it is possible to recover the probability of the answer under the inaccessible categorical distribution of the LLM. 
    \item  We propose a theoretical model for multiple rephrased queries on a simplified top-k decoding setting \citep{holtzman2019curious}. Our analysis implies that generating multiple answers using the same base query and top-k decoding can also recover a tempered version of the probability of the answer under the inaccessible categorical distribution of the LLM. While generating multiple answers in this way (without rephrasing) might be sufficient for good calibration, we find that rephrasing results in additional tempering of the resulting uncertainty estimate, which is known to improve calibration. 
    \item In practice however, when comparing top-k and top-1 decoding with and without rephasing, in terms of Brier score, we find that top-1 decoding with rephrasing results in the best trade-off between accuracy and calibration.
    %\item  Our analysis also uncovers limitations in the usefullness of rephrasings in the top-k decoding setting. In particular, we expect that, if post-hoc tempering or access to the internal LLM temperature is available, supplementing top-k decoding with rephasings will result in limited to no calibration gains.

\end{itemize}

\section{Rephrasing drastically improves calibration for top-1 decoding}

\begin{table}[t]
\centering
\captionsetup{skip=8pt}
\begin{tabularx}{\textwidth}{c|X}
\hline
\textbf{Method} & \textbf{Question} \\
\hline\hline
original & What part of the digestive system first causes chemical changes to food?
A. Teeth in the mouth. B. Saliva in the mouth. C. Enzymes in the stomach.
D. Enzymes in the small intestine. \\
\hline\hline
reword & Which region of the gastrointestinal tract initiates the initial chemical modifications
to food intake? A. Teeth in the mouth. B. Saliva in the mouth. C. Enzymes in the stomach.
D. Enzymes in the small intestine.\\
\hline
rephrase & In what region of the digestive system does the food undergo its initial chemical
transformations? A. Teeth in the mouth. B. Saliva in the mouth. C. Enzymes in the stomach.
D. Enzymes in the small intestine.\\
\hline
paraphrase &  At what point in the digestive process do initial chemical transformations of food
occur and which section of the system carries out this function? A. Teeth in the mouth. B. Saliva in the mouth. C. Enzymes in the stomach.
D. Enzymes in the small intestine.\\
\hline
expansion & Considering the intricate process by which our bodies break down and absorb nutrients from food, which specific organ or region within the digestive system initiates the essential biochemical transformations through enzyme secretion and the beginning of the digestion process? A. Teeth in the mouth. B. Saliva in the mouth. C. Enzymes in the stomach.
D. Enzymes in the small intestine. \\
\hline
\end{tabularx}
\caption{Rephrasing examples generated by Mistral-7B, with rephrasing methods listed on the left and corresponding rephrases on the right.}
\label{table:generation_mistral}
\end{table}

Let $f: \mathcal{X} \rightarrow \mathcal{Y}$ be an LLM which takes $\bx$ an input query in the form of a multiple choice question, and outputs $y$, an answer. \emph{We first consider top-1 decoding such that the answers of the LLM are deterministic.} We consider randomized transformations of the base query $\mathcal{T}(\bx)\sim \tau$ in the form of rephrasings of the query, and the most probable answer under the transformations $A = \mathrm{argmax}_i \mathbb{P}\left( f(\mathcal{T}(\bx)) = i\right)$. In a multiple choice question setting (which can be seen as a multi-class classification problem), we will use $A$ as the predicted class and 
\begin{equation*}
p_A(\bx) = \mathbb{P}\left( f(\mathcal{T}(\bx)) = A\right),
\end{equation*}
as our confidence about this prediction (here the predicted class coincides with a predicted token denoting this class). We consider four types of rephrasings, with an increasing level of modification to the original query:
\begin{itemize}
\item Reword: Focuses on replacing words with their synonyms without significantly altering the sentence structure or adding new content.
\item Rephrase: Modifies the original question with changes in structure and possibly synonyms to achieve a similar but distinct question.
\item Paraphrase: Reconstructs the original query, often significantly, to retain its meaning while altering its presentation. 
% This can involve changing the sentence structure, using synonyms, and sometimes modifying the level of detail.
\item Expansion: Elaborates on the original query, making it more detailed or specific, often by adding context or additional considerations.
\end{itemize}
We provide our one-shot prompt template for each rephrasing method in Table~\ref{table:prompt} in Appendix~\ref{app:prompt}, and example generations from Mistral-7B in Table~\ref{table:generation_mistral} and generations from Llama-7B/13B in Appendix~\ref{app:generations}. In general, we perform the rephrasings with a separate instance of the same model that responds to the queries. We estimate $p_A(\bx)$ using Monte Carlo sampling with 10 draws from $\mathcal{T}(\bx)\sim \tau$ to estimate uncertainty with our method unless stated otherwise.

\begin{table*}[t]
    \caption{Evaluation results on ARC-Challenge with various rephrasing methods applied to three LLMs. In the majority of cases, the rephrasing approach outperforms the naive baseline by $10-40\%$ in AUROC, $10-30\%$ in ECE and $0-0.4$ in Brier.
    % $10-40\%$ in AUROC, $40-80\%$ in ECE and $0-0.3$ in Brier.
    }
    \label{arc-challenge-main}
\centering
\small
\begin{tabular}{cccccccc}
\toprule
Model & Rephrasing & Acc $\uparrow$ & ECE $\downarrow$ & TACE $\downarrow$ & Brier $\downarrow$  & AUROC $\uparrow$ & temp  \\
\toprule
\multirow {6}{*}{Mistral-7B} 
& {top-1}        & 0.742 & 0.258 & \textbf{0.065} & 0.517 & 0.5   & -    \\
& {hint} &  0.593, & 0.201, & 0.108, & 0.614, & 0.695, & - \\ 
\cdashline{2-8} \rule{0pt}{2.25ex}
& {reword}       & 0.619 & 0.12  & 0.103 & 0.512 & \textbf{0.846} & 1.0  \\
& {rephrase}      & 0.555 & 0.125 & 0.103 & 0.571 & 0.817 & 1.5  \\
& {paraphrase}   & 0.525 & \textbf{0.102} & 0.115 & 0.592 & 0.827 & 1.5  \\
& {expansion}    & 0.602 & 0.133 & {0.099} & \textbf{0.509} & 0.847 & 1.0  \\
\midrule
\multirow {6}{*}{Llama-2-7B} & {top-1} &   0.483 & 0.517 & - & 1.034 & 0.5 & - \\
& {hint} &  0.258, & \textbf{0.071}, & \textbf{0.144}, & 0.839, & 0.562, & - \\ 
\cdashline{2-8} \rule{0pt}{2.25ex}
& {reword}       & 0.352 & 0.193 & 0.176 & 0.853 & 0.626 & 1.5  \\
& {rephrase}      & 0.381 & 0.263 & 0.173 & 0.871 & 0.656 & 1.5  \\
& {paraphrase}   & 0.39  & 0.287 & 0.162 & 0.883 & 0.67  & 1.0  \\
& {expansion}    & 0.373 & 0.112 & 0.153 & \textbf{0.778} & \textbf{0.687} & 1.5  \\
\midrule
\multirow {6}{*}{Llama-2-13B} & {top-1} &  0.508 & 0.492 & - & 0.983 & 0.5 & - \\
& {hint} &  0.331, & 0.147, & 0.134, & 0.813, & 0.57, & - \\ 
\cdashline{2-8} \rule{0pt}{2.25ex}
& {reword}      & 0.445 & \textbf{0.084} & \textbf{0.119} & \textbf{0.714} & 0.721 & 1.5  \\
& {rephrase}     & 0.441 & 0.128 & 0.134 & 0.727 & 0.713 & 1.5  \\
& {paraphrase}  & 0.453 & 0.092 & 0.129 & 0.717 & 0.697 & 1.5  \\
& {expansion}   & 0.441 & 0.154 & 0.142 & \textbf{0.715} & \textbf{0.784} & 1.2  \\
\bottomrule
\end{tabular}
\end{table*}

\begin{table*}[t]
    \caption{Evaluation results on ARC-Easy with various rephrasing methods applied to three LLMs. In the majority of cases, the rephrasing approach outperforms the naive baseline by $10-40\%$ in AUROC, $10-30\%$ in ECE, and $0-0.4$ in Brier.
    % $10-40\%$ in AUROC, $40-80\%$ in ECE and $0-0.2$ in Brier
    }
    \label{arc-easy-main}
\centering
\small
\begin{tabular}{cccccccc}
\toprule
Model & Rephrasing & Acc $\uparrow$ & ECE $\downarrow$ & TACE $\downarrow$ & Brier $\downarrow$  & AUROC $\uparrow$ & temp  \\
\toprule
\multirow {6}{*}{Mistral-7B} & {top-1}             & 0.866 & 0.134 & \textbf{0.034} & \textbf{0.269} & 0.5   & -   \\
& {hint} &  0.773, & 0.17, & 0.076, & 0.386, & 0.795, & - \\ 
\cdashline{2-8} \rule{0pt}{2.25ex}
& {reword}            & 0.753 & 0.045 & 0.062 & 0.297 & 0.931 & 1.0 \\
& {rephrase}           & 0.678 & 0.035 & 0.076 & 0.357 & \textbf{0.953} & 1.5 \\
& {paraphrase}        & 0.663 & 0.036 & 0.08  & 0.381 & 0.943 & 1.5 \\
& {expansion}         & 0.742 & \textbf{0.034} & 0.067 & 0.31  & 0.936 & 1.0 \\
\midrule
\multirow {6}{*}{Llama-2-7B} & {top-1}             & 0.672 & 0.328 & 0.082 & 0.656 & 0.5   & -   \\
& {hint} &  0.231, & \textbf{0.041}, & 0.149, & 0.827, & 0.663, & - \\  
\cdashline{2-8} \rule{0pt}{2.25ex}
& {reword}            & 0.43  & 0.084 & \textbf{0.119} & 0.672 & 0.818 & 1.5 \\
& {rephrase}           & 0.535 & 0.131 & \textbf{0.117} & \textbf{0.603} & 0.830  & 1.5 \\
& {paraphrase}        & 0.526 & 0.184 & 0.125 & 0.626 & \textbf{0.831} & 1.0 \\
& {expansion}         & 0.405 & 0.045 & \textbf{0.119} & 0.692 & 0.818 & 1.5 \\
\midrule
\multirow {6}{*}{Llama-2-13B} & {top-1}            & 0.617 & 0.383 & \textbf{0.096} & 0.767 & 0.5   & -   \\
& {hint} &  0.346, & \textbf{0.089}, & 0.128, & 0.77, & 0.673, & - \\  
\cdashline{2-8} \rule{0pt}{2.25ex}
& {reword}           & 0.546 & 0.07  & 0.11  & 0.58  & 0.814 & 1.5 \\
& {rephrase}          & 0.526 & 0.07  & 0.112 & 0.579 & 0.842 & 1.5 \\
& {paraphrase}       & 0.518 & 0.104 & 0.119 & 0.604 & 0.815 & 1.5 \\
& {expansion}        & 0.524 & 0.078 & 0.12  & \textbf{0.552} & \textbf{0.893} & 1.2 \\
\bottomrule
\end{tabular}
\end{table*}

\begin{table*}[t]
    \caption{Evaluation results on OpenBookQA with various rephrasing methods applied to three LLMs. In the majority of cases, the rephrasing approach outperforms the naive baseline by $10-40\%$ in AUROC, $10-30\%$ in ECE, and $0-0.4$ in Brier.
    % $10-40\%$ in AUROC, $40-80\%$ in ECE and $0-0.3$ in Brier.
    }
    \label{openbookqa-main}
\centering
\small
\begin{tabular}{cccccccc}
\toprule
Model & Rephrasing & Acc $\uparrow$ & ECE $\downarrow$ & TACE $\downarrow$ & Brier $\downarrow$  & AUROC $\uparrow$ & temp  \\
\toprule
\multirow {6}{*}{Mistral-7B} & {top-1}           & 0.655 & 0.345 & \textbf{0.086} & 0.69 & 0.5    & -   \\ 
& {hint} &  0.56, & 0.265, & 0.119, & 0.71, & 0.606, & - \\ 
\cdashline{2-8} \rule{0pt}{2.25ex}
& {reword}          & 0.552 & 0.105 & 0.102 & \textbf{0.592} & 0.796 & 1.0 \\ 
& {rephrase}         & 0.482 & 0.107 & 0.122 & 0.641 & 0.809 & 1.5 \\ 
& {paraphrase}      & 0.49  & \textbf{0.076} & 0.116 & 0.622 & 0.826 & 1.5 \\ 
& {expansion}       & 0.518 & 0.087 & 0.117 & 0.596 & \textbf{0.837} & 1.0 \\ 
\midrule
\multirow {6}{*}{Llama-2-7B} & {top-1}           & 0.478 & 0.522 & \textbf{0.131} & 1.045 & 0.5   & -   \\ 
& {hint} &  0.275, & \textbf{0.08}, & 0.142, & 0.832, & 0.556, & - \\ 
\cdashline{2-8} \rule{0pt}{2.25ex}
& {reword}          & 0.388 & 0.137 & 0.143 & 0.786 & 0.689 & 1.5 \\ 
& {rephrase}         & 0.39  & 0.196 & 0.156 & 0.806 & \textbf{0.721} & 1.5 \\ 
& {paraphrase}      & 0.398 & 0.227 & 0.159 & 0.834 & 0.712 & 1.0 \\ 
& {expansion}       & 0.362 & 0.083 & 0.138 & \textbf{0.775} & 0.678 & 1.5 \\ 
\midrule
\multirow {6}{*}{Llama-2-13B} & {top-1}          & 0.418 & 0.582 & -     & 1.165 & 0.5   & -   \\ 
& {hint} &  0.295, & \textbf{0.069}, & \textbf{0.138}, & 0.809, & 0.613, & - \\
\cdashline{2-8} \rule{0pt}{2.25ex}
& {reword}         & 0.428 & 0.117 & 0.142 & 0.75  & 0.676 & 1.5 \\ 
& {rephrase}        & 0.428 & 0.095 & 0.14  & \textbf{0.729} & 0.73  & 1.5 \\ 
& {paraphrase}     & 0.41  & 0.116 & 0.141 & 0.759 & 0.682 & 1.5 \\
& {expansion}      & 0.41  & 0.143 & 0.147 & 0.772 & \textbf{0.702} & 1.2 \\ 
\bottomrule
\end{tabular}
\end{table*}

We used three different models, the Llama-2 7B model, the Llama-2 13B model \citep{touvron2023llama} and the Mistral 7B model \citep{jiang2023mistral}. We tested our framework on three multiple choice tasks of different difficulty namely ARC-Challenge, ARC-Easy \citep{allenai:arc}, and Openbookqa \citep{mihaylov2018can}. Following \citet{kojima2022large}, we extract the answer from LLM-generated texts by looking at the first appearance of A/B/C/D. To test for calibration we used standard calibration metrics, including the ECE and TACE \citep{naeini2015obtaining},  Brier score \citep{murphy1973new} and AUROC \citep{murphy2012machine}. We note that for a fair comparison when the accuracy drops significantly, we must consult the Brier score which is a proper scoring rule. This is because, the ECE, TACE and AUROC are not proper scoring rules and can in general trade-off accuracy for calibration. For a baseline, we assumed $100\%$ confidence for each deterministic prediction. We also tested the "hint" based approach of \cite{xiong2023can}, which we describe in detail in Appendix \ref{app:prompt}. 

We present the results in Tables \ref{arc-challenge-main}, \ref{arc-easy-main} and \ref{openbookqa-main}. In the majority of cases rephrasing outperforms the naive baseline by $10-40\%$ in AUROC, $10-30\%$ in ECE, and $0-0.4$ in Brier. Our approach also typically outperforms the ``hint" base approach of \cite{xiong2023can} by $10-20\%$ in AUROC, $5-10\%$ in ECE, and $0.1$ in Brier. In particular, the "hint" based approach which more inflexible than our approach and typically hurts accuracy significantly $10-20\%$ compared to $5-10\%$ for our approach. For our method, these accuracy drops are more prevalent in the smaller 7B models, while the larger 13B model often shows a much smaller drop. 

 Crucially, the different rephrasing methods exhibit different calibration gains. On average, in terms of all calibration metrics the best methods are the "expansion" and "reword" methods, which make the queries more verbose, and substitute words with synonyms respectively. In terms of AUROC "expansion" outperforms the alternatives by $1-5\%$. In terms of the Brier score it outperforms by $\approx 0.05$. To instantiate our rephrasings we used a prompt with a one-shot example and a temperature parameter resulting in greater or smaller varieties of rephrasings. We include this temperature parameter in the Tables. Generally, we choose this temperature that balances accuracy and calibration. In Figure \ref{fig:error_sources_complete} we plot the behaviour as the number of MC draws increases.

 In Appendix \ref{cot_experiments}, we also compare with Chain-of-Thought \cite{wei2022chain} for uncertainty estimation. We find that we get competitive results with CoT. At the same time our method is significantly easier and more natural to implement for humans interacting via text with an LLM.

\begin{figure*}[t!]
    \centering
    
    \begin{subfigure}{.2\textwidth}
    \includegraphics[width=\textwidth]{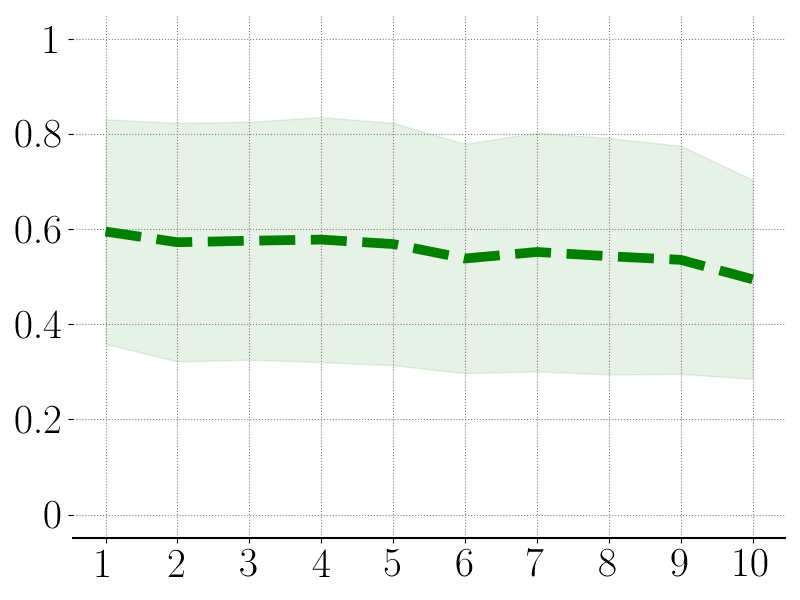}\\
    \caption{Accuracy}
    \label{fig:subfig1redun}
    \end{subfigure}%
    \begin{subfigure}{.2\textwidth}
    \includegraphics[width=\textwidth]{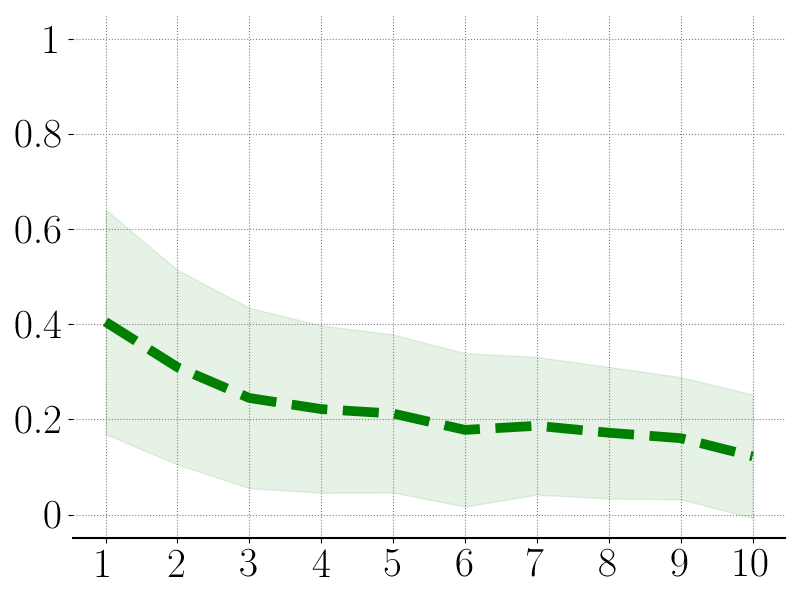}\\
    \caption{ECE}
    \label{fig:subfig2redun}
    \end{subfigure}%
    \begin{subfigure}{.2\textwidth}
    \includegraphics[width=\textwidth]{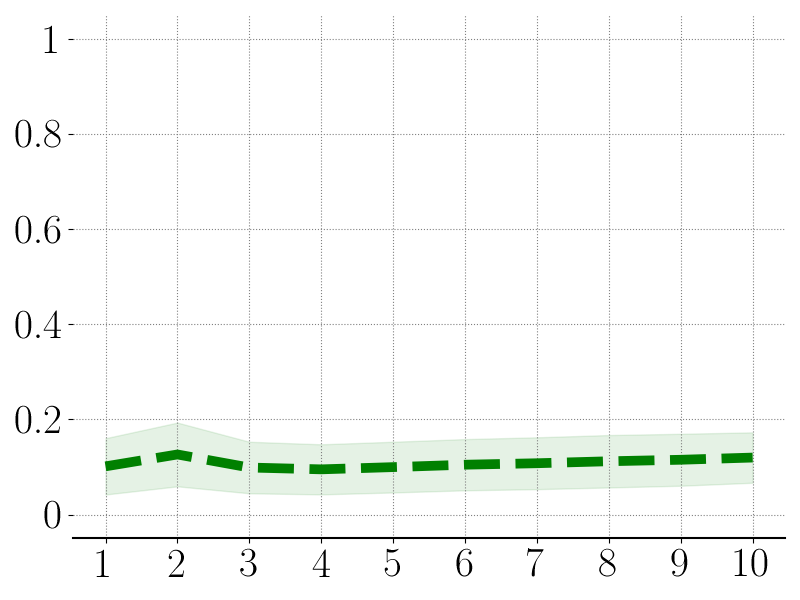}\\
    \caption{TACE}
    \label{fig:subfig3redun}
    \end{subfigure}%
    \begin{subfigure}{.2\textwidth}
    \includegraphics[width=\textwidth]{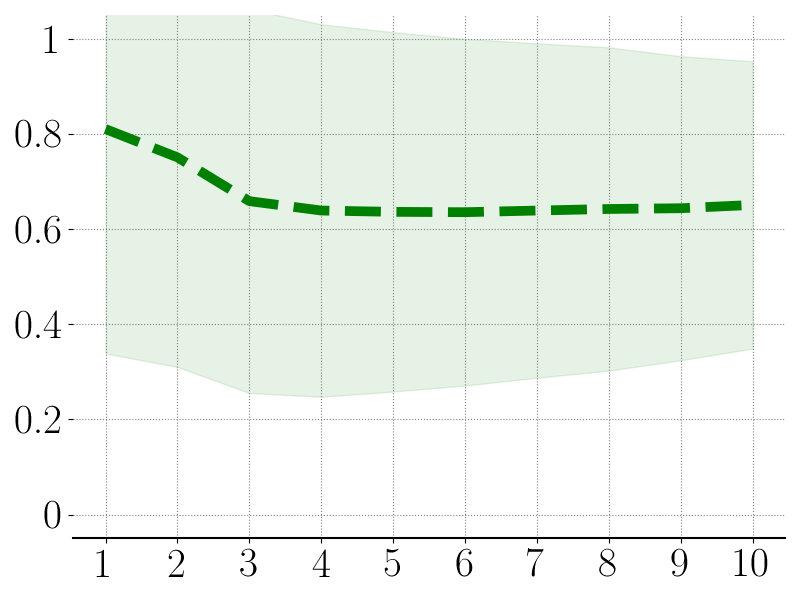}\\
    \caption{Brier}
    \label{fig:subfig4redun}
    \end{subfigure}%
    \begin{subfigure}{.2\textwidth}
    \includegraphics[width=\textwidth]{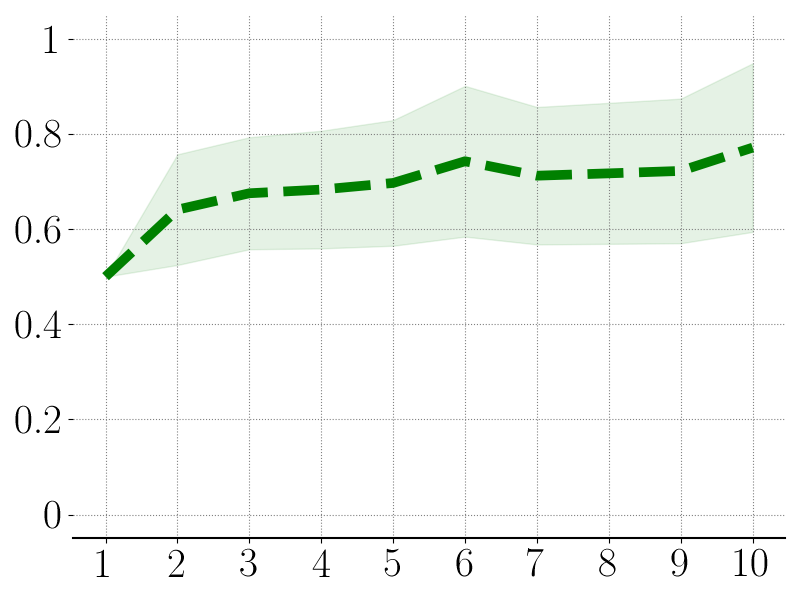}\\
    \caption{AUROC}
    \label{fig:subfig5redun}
    \end{subfigure}%
    \caption{The behavior of the Accuracy, ECE, TACE, Brier, and AUROC for all datasets, architectures, and expansion methods, as we increase the number of samples. We plot the average value as well as confidence intervals $\pm2\sigma$. We see that the ECE and the AUROC improve with more samples while the accuracy drops slightly. This might be because the meaning of some queries is completely destroyed by our rephrasings. The Brier score captures this tradeoff by having a minimum at approximately 5 samples. The TACE remains relatively stable with respect to the number of samples.}
    \label{fig:error_sources_complete}
\end{figure*}
\begin{table*}[t]
    \caption{Comparisons between our rephrasing methods and white-box logit uncertainty estimation. We see that our rephrasing methods achieve similar calibration to what would be achieved if we had access to last layer logits. This is evident both in the AUROC and TACE as well as the Brier score, which also accounts for accuracy.
    }
    \label{logits_vs_rephrasing-main}
\centering
\small
\begin{tabular}{cccccccc}
\toprule
Dataset & Model & Method & Acc $\uparrow$ & ECE $\downarrow$ & TACE $\downarrow$ & Brier $\downarrow$  & AUROC $\uparrow$\\
\toprule
\multirow {6}{*}{ARC-C} & \multirow {2}{*}{Mistral-7B} & {logits}      & 0.742 & 0.252 & 0.075 & 0.503 & 0.741 \\ 
& & {expansion}                                                        & 0.602 & 0.133 & 0.099 & 0.509 & 0.847 \\ 
\cdashline{2-8} \rule{0pt}{2.25ex}
& \multirow {2}{*}{Llama-2-7B} & {logits}                              & 0.483 & 0.362 & 0.168 & 0.853 & 0.621 \\ 
& & {expansion}                                                        & 0.373 & 0.112 & 0.153 & 0.778 & 0.687 \\ 
\cdashline{2-8} \rule{0pt}{2.25ex}
& \multirow {2}{*}{Llama-2-13B} & {logits}                             & 0.508 & 0.132 & 0.141 & 0.704 & 0.669 \\ 
& & {reword}                                                           & 0.445 & 0.084 & 0.119 & 0.714 & 0.721 \\ 
\midrule
\multirow {6}{*}{ARC-E} & \multirow {2}{*}{Mistral-7B} & {logits}      & 0.866 & 0.128 & 0.037 & 0.264 & 0.818 \\ 
& & {reword}                                                           & 0.753 & 0.045 & 0.062 & 0.297 & 0.931 \\ 
\cdashline{2-8} \rule{0pt}{2.25ex}
& \multirow {2}{*}{Llama-2-7B} & {logits}                              & 0.672 & 0.190 & 0.098 & 0.493 & 0.779 \\ 
& & {rephrase}                                                         & 0.535 & 0.131 & 0.117 & 0.603 & 0.830 \\ 
\cdashline{2-8} \rule{0pt}{2.25ex}
& \multirow {2}{*}{Llama-2-13B} & {logits}                             & 0.617 & 0.060 & 0.094 & 0.498 & 0.763 \\ 
& & {expansion}                                                        & 0.524 & 0.078 & 0.12  & 0.552 & 0.893 \\ 
\midrule
\multirow {6}{*}{OBQA} & \multirow {2}{*}{Mistral-7B} & {logits}       & 0.655 & 0.298 & 0.085 & 0.602 & 0.705 \\ 
& & {reword}                                                           & 0.552 & 0.105 & 0.102 & 0.592 & 0.796 \\
\cdashline{2-8} \rule{0pt}{2.25ex}
& \multirow {2}{*}{Llama-2-7B} & {logits}                              & 0.478 & 0.277 & 0.147 & 0.758 & 0.642 \\ 
& & {expansion}                                                        & 0.362 & 0.083 & 0.138 & 0.775 & 0.678 \\ 
\cdashline{2-8} \rule{0pt}{2.25ex}
& \multirow {2}{*}{Llama-2-13B} & {logits}                             & 0.418 & 0.168 & 0.135 & 0.723 & 0.650 \\
& & {rephrase}                                                         & 0.428 & 0.095 & 0.14  & 0.729 & 0.73  \\ 
\bottomrule
\end{tabular}
\end{table*}
\section{Rephrasing works as well as having access to the last layer logits}

%We now explore possible explanations for the differences in performance among the different transformations. 
% We start by deriving a result  similar to Proposition \ref{main_theorem} for an \textit{unknown noise distribution in latent space} $\rho$, and then elucidate two important sources of error.
We now derive a proposition that elucidates why $p_A(\bx)$ results in calibrated estimates of uncertainty.

\begin{proposition}\label{theorem_general_noise}
Let $f : \mathcal{X} \rightarrow \mathcal{Y}$ be an LLM, $\bx$ is a base query and $\mathcal{T}(\bx)\sim \tau$ is some randomized transformation of the base query. Let
\begin{equation}
p_A(\bx) = \mathbb{P}\left( f(\mathcal{T}(\bx)) = A\right),
\end{equation}
be the probability of sampling the most probable answer $A \in \mathcal{Y}$ under transformations $\mathcal{T}(\bx)\sim \tau$. Let $\bz_{mean}+\epsilon_{rephrase}$ be the latent representation of $\bx$ under $\mathcal{T}(\bx)$ at the final LLM layer, where $\bz_{mean}$ is the mean representation and $\epsilon_{rephrase}$ is some additive noise. Let $\bw$ be the separating hyperplane between the most probable answer $A$ and the second most probable answer $B$. Assuming that $\bw^{\top} \epsilon_{rephrase}\sim \rho$ follows a logistic distribution with $\mu=0$ and $s=1$ then
\begin{equation}
\begin{split}
    p_A(\bx) = p(A\vert \bz_{mean} , f)
\end{split}
\end{equation}
where $p(A\vert \bz_{mean} , f)$ is the probability of $A$ given $\bz_{mean}$ under the categorical distribution of the final layer.
\end{proposition}

We prove the above for the binary case of two classes $A$ and $B$ in Appendix~\ref{app:proof}, but expect that it should be sufficiently informative in multi-class settings when $A,B$ are much more probable than other classes. A crucial assumption for recovering well-calibrated predictions is that $\bw^{\top} \epsilon_{rephrase}\sim \rho$ follows a logistic distribution with $\mu=0$ and $s=1$. We test this assumption by computing the cumulative of $\rho$ for our different experimental setups. In Figure \ref{fig:subfig3redun-topk} we find and plot the empirical cumulative using a Kolmogorov-Smirnov test \citep{smirnov1948table} and $S=100$ MC samples of $\rho$ for Mistral-7B, ARC-Challenge, and the ``expansion" rephrasing method. We see that the indeed the cumulative is approximately logistical validating our prediction (the confidence bands cover different queries $\bx$). In Table \ref{logits_vs_rephrasing-main} we use the logits of the answers as an oracle white-box uncertainty estimate. Specifically, we apply the softmax function and use the probability of the most probable class as our estimate of uncertainty. We compare the results of this method with the best rephrasing method (in terms of Brier) from Tables \ref{arc-challenge-main}, \ref{arc-easy-main} and \ref{openbookqa-main}. We see observe that our uncertainty estimates that are similar to what we would get if we had access to the last layer logits.

\section{For top-k decoding, rephrasing tempers predictive uncertainty}

\begin{figure*}[t!]
    \centering
    
    \begin{subfigure}{.25\textwidth}
    \includegraphics[width=\textwidth]{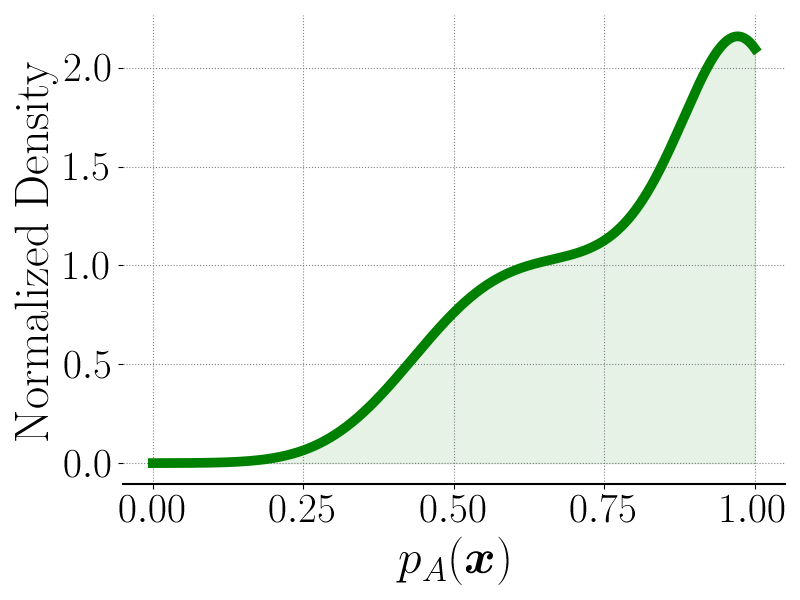}\\
    \caption{$p_A(\bx)$ for top-k without rephrasing}
    \label{fig:subfig1redun-topk}
    \end{subfigure}%
    \hspace{15pt}
    \begin{subfigure}{.25\textwidth}
    \includegraphics[width=\textwidth]{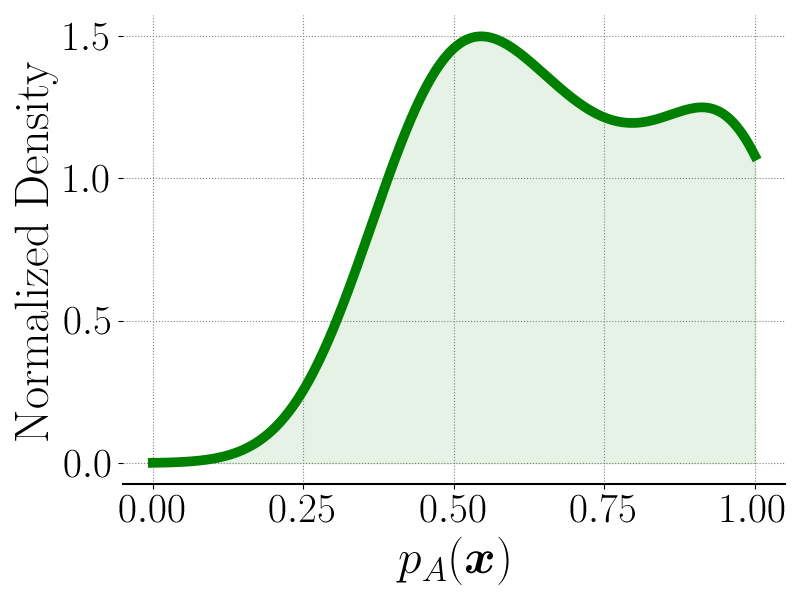}\\
    \caption{$p_A(\bx)$ for top-k with rephrasing}
    \label{fig:subfig2redun-topk}
    \end{subfigure}%
    \hspace{15pt}
    \begin{subfigure}{.25\textwidth}
    \includegraphics[width=\textwidth]{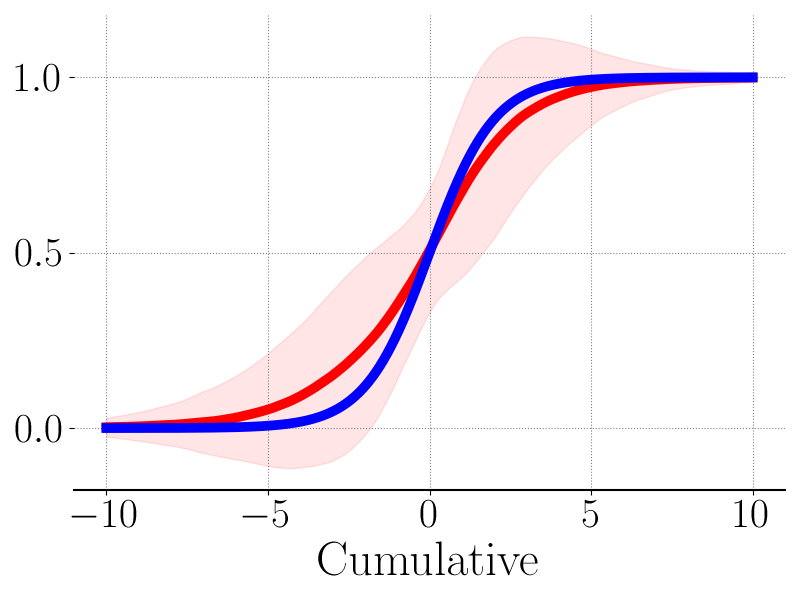}\\
    \caption{Logistic (blue), and empirical cdf (red)}
    \label{fig:subfig3redun-topk}
    \end{subfigure}%
    \caption{We plot the distribution of $p_A(\bx)$ for the case of top-k decoding with and without rephrasing, for all datasets, models, and rephrasing methods. We see that rephrasing primarily acts to temper the probability of the most probable class $A$, thus making the model less confident and possibly better calibrated. We also plot the logistic (blue), and empirical cdf (red) for $\bw^{\top} \epsilon_{rephrase}\sim \rho$ for Mistral-7B, ARC-Challenge, and the ``expansion" rephrasing method for top-1 decoding. $\rho$ is often close to a logistic distribution.}
    \label{fig:topk}
\end{figure*}
\begin{table*}[!t]
    \caption{Evaluation results on ARC-Challenge with various rephrasing methods applied to three LLMs using top-k decoding. In the majority of cases rephrasing + top-k outperforms simple top-k in terms of calibration. }
    \label{topk-chal}
\centering
\small
\begin{tabular}{cccccccc}
\toprule
Model & Rephrasing & Acc $\uparrow$ & ECE $\downarrow$ & TACE $\downarrow$ & Brier $\downarrow$  & AUROC $\uparrow$ & temp  \\
\toprule
 \multirow {5}{*}{Mistral-7B} & {top-k} & 0.746, & 0.272, & 0.091,     & 0.511, & 0.6, & -\\ 
& {temp-sampling} & 0.742 & 0.272 & 0.089 & 0.513 & 0.605 & - \\
 \cdashline{2-8} \rule{0pt}{2.25ex}
 & {reword} & 0.547, & \textbf{0.05}, & 0.093,      & 0.543, & \textbf{0.864}, & 1.5\\  
 & {rephrase} & 0.64, & 0.106, & \textbf{0.086},     & \textbf{0.485}, & 0.82, & 1.0\\  
 & {paraphrase} & 0.631, & 0.11, & 0.098,  & 0.495, & 0.83, & 1.0\\ 
 & {expansion} & 0.517, & 0.061, & 0.114,  & 0.573, & 0.859, & 1.5\\  
\midrule
\multirow {5}{*}{Llama-2-7B} & {top-k} & 0.436, & 0.201, & \textbf{0.139},     & \textbf{0.761}, & 0.602, & -\\
& {temp-sampling} & 0.441 & 0.211 & 0.132 & 0.757 & 0.621 & - \\
\cdashline{2-8} \rule{0pt}{2.25ex}
  & {reword} & 0.335, & 0.187, & 0.166,     & 0.858, & 0.62, & 1.5\\ 
  & {rephrase} & 0.356, & 0.314, & 0.17,     & 0.944, & 0.627, & 1.0\\  
  & {paraphrase} & 0.309, & 0.185, & 0.162, & 0.851, & \textbf{0.69}, & 1.5\\
  & {expansion} & 0.322, & \textbf{0.144}, & 0.155,  & 0.828, & 0.622, & 1.5\\  
\midrule
\multirow {5}{*}{Llama-2-13B} & {top-k} & 0.462, & 0.125, & \textbf{0.115},     & \textbf{0.679}, & \textbf{0.753}, & -\\ 
& {temp-sampling} & 0.47, & 0.122 & 0.115 & 0.662 & 0.766 & - \\
\cdashline{2-8} \rule{0pt}{2.25ex}
 & {reword} & 0.352, & 0.087, & 0.136,     & 0.771, & 0.687, & 1.5\\  
 & {rephrase} & 0.398, & \textbf{0.068}, & 0.136,    & 0.725, & 0.743, & 1.0\\  
 & {paraphrase} & 0.364, & 0.109, & 0.137, & 0.738, & 0.719, & 1.2\\ 
 & {expansion} & 0.373, & 0.124, & 0.143,  & 0.76,  & 0.669, & 1.5\\  
\bottomrule
\end{tabular}
\end{table*}
In practice, the assumptions of the above proposition are too restrictive. In particular, decoding in LLMs is performed with top-k decoding or nucleus sampling instead of top-1 decoding. Furthermore while for an oracle choice of the rephrasing intensity the modeling assumption that $\bw^{\top} \epsilon_{\eta}\sim \rho$ follows a logistic distribution with $\mu=0$ and $s=1$ might be correct, in general, the variance of the noise in latent space is unknown. It is thus illustrative to consider an extension of our toy model. The following proposition explores these extensions.

\begin{proposition}\label{theorem_topk_temperature_scaling}
Let $g : \mathbb{R}^{d_{\eta}} \rightarrow \mathcal{Y}$ be the final encoder layer of an LLM, $\bx$ is a base query and $\mathcal{T}(\bx)\sim \tau$ is some randomized transformation of the base query. Let
\begin{equation}
p_A(\bx) = \mathbb{P}\left( f(\mathcal{T}(\bx)) = A\right),
\end{equation}
be the probability of sampling the most probable answer $f(\bx)=A \in \mathcal{Y}$ under transformations $\mathcal{T}(\bx)\sim \tau$. Let $\bz_{mean}+\epsilon_{topk}+\epsilon_{rephrase}$ be the latent representation of $\bx$ under $\mathcal{T}(\bx)$ at the final LLM layer, where $\bz_{mean}$ is the mean representation and $\epsilon_{topk}$ is additive noise resulting from the top-k decoding and $\epsilon_{rephrase}$ is additive noise resulting from the rephrasings $\mathcal{T}(\bx)$. Assuming that $\bw^{\top} (\epsilon_{topk}+\epsilon_{rephrase})\sim \rho$ approximately follows a logistic distribution with $\mu=0$ and $s=\sqrt{s^2_{topk}+s^2_{rephrase}}$ then
\begin{equation}
\begin{split}
     p_A(\bx) \approx 0.5+\frac{1}{\sqrt{s^2_{topk}+s^2_{rephrase}}}(p(A\vert \bz_{mean} , f)-0.5)
\end{split}
\end{equation}
where $p(A\vert \bz_{mean} , f)$ is the probability of $A$ given $\bz_{mean}$ under the categorical distribution of $g$.
\end{proposition}

\begin{table*}[!t]
    \caption{Evaluation results on ARC-Easy with various rephrasing methods applied to three LLMs using top-k decoding. In the majority of cases rephrasing + top-k outperforms simple top-k in terms of calibration.}
    \label{topk-easy}
\centering
\small
\begin{tabular}{cccccccc}
\toprule
Model & Rephrasing & Acc $\uparrow$ & ECE $\downarrow$ & TACE $\downarrow$ & Brier $\downarrow$  & AUROC $\uparrow$ & temp  \\
\toprule
 \multirow {6}{*}{Mistral-7B} & {top-k} & 0.868, & 0.133, & \textbf{0.042},     & \textbf{0.255}, & 0.695, & -\\
& {temp-sampling} & 0.859 & 0.131 & 0.046 & 0.266 & 0.677 & - \\
 \cdashline{2-8} \rule{0pt}{2.25ex}
 & {reword} & 0.694, & 0.054, & 0.076,     & 0.344, & 0.941, & 1.5\\ 
 & {rephrase} & 0.789, & 0.047, & 0.049,    & 0.274, & 0.911, & 1.0\\ 
 & {paraphrase} & 0.753, & \textbf{0.036}, & 0.056, & 0.3, & 0.922, & 1.0\\ 
 & {expansion} & 0.63, & 0.042, & 0.086,   & 0.403, & \textbf{0.942}, & 1.5\\ 
\midrule
\multirow {6}{*}{Llama-2-7B} & {top-k} & 0.612, & 0.25, & 0.115,      & 0.612, & 0.73, & -\\
& {temp-sampling} & 0.619 & 0.261 & 0.114 & 0.617 & 0.717 & - \\
\cdashline{2-8} \rule{0pt}{2.25ex}
  & {reword} & 0.401, & \textbf{0.074}, & 0.121,     & 0.681, & 0.825, & 1.5\\ 
  & {rephrase} & 0.564, & 0.145, & \textbf{0.108},    & \textbf{0.584}, & 0.819, & 1.0\\ 
  & {paraphrase} & 0.425, & 0.08, & 0.117,  & 0.665, & \textbf{0.835}, & 1.5\\  
  & {expansion} & 0.335, & 0.054, & 0.138,  & 0.742, & 0.791, & 1.5\\ 
\midrule
\multirow {6}{*}{Llama-2-13B} & {top-k} & 0.557, & 0.06, & \textbf{0.098},     & \textbf{0.528}, & \textbf{0.865}, & -\\
& {temp-sampling} & 0.544 & 0.087 & 0.107 & 0.532 & 0.866 & - \\
\cdashline{2-8} \rule{0pt}{2.25ex}
 & {reword} & 0.412, & 0.106, & 0.129,    & 0.72, & 0.741, & 1.5\\ 
 & {rephrase} & 0.458, & \textbf{0.05}, & 0.12,     & 0.643, & 0.817, & 1.0\\ 
 & {paraphrase} & 0.427, & 0.066, & 0.126,& 0.652, & 0.845, & 1.2\\  
 & {expansion} & 0.366, & 0.087, & 0.13,  & 0.74, & 0.75, & 1.5\\ 
\bottomrule
\end{tabular}
\end{table*}

\begin{table*}[!t]
    \caption{Evaluation results on OpenBookQA with various rephrasing methods applied to three LLMs using top-k decoding. In the majority of cases rephrasing + top-k outperforms simple top-k in terms of calibration.}
    \label{topk-opqa}
\centering
\small
\begin{tabular}{cccccccc}
\toprule
Model & Rephrasing & Acc $\uparrow$ & ECE $\downarrow$ & TACE $\downarrow$ & Brier $\downarrow$  & AUROC $\uparrow$ & temp  \\
\toprule
 \multirow {5}{*}{Mistral-7B} & {top-k} & 0.638, & 0.289, & 0.101,     & 0.636, & 0.636, & -\\
& {temp-sampling} & 0.668 & 0.289 & 0.098 & 0.607 & 0.624 & - \\
 \cdashline{2-8} \rule{0pt}{2.25ex}
 & {reword} & 0.528, & 0.103, & 0.105,     & 0.606, & 0.794, & 1.5\\  
 & {rephrase} & 0.582, & 0.109, & \textbf{0.093},    & \textbf{0.542}, & \textbf{0.821}, & 1.0\\  
 & {paraphrase} & 0.552, & 0.078, & 0.101, & 0.57, & 0.817, & 1.0\\  
 & {expansion} & 0.445, & \textbf{0.061}, & 0.128,  & 0.653, & 0.818, & 1.5\\  
\midrule
\multirow {5}{*}{Llama-2-7B} & {top-k} & 0.412, & 0.208, & \textbf{0.129},     & \textbf{0.776}, & 0.617, & -\\
& {temp-sampling} & 0.442 & 0.235 & 0.13  & 0.772 & 0.599 & - \\
\cdashline{2-8} \rule{0pt}{2.25ex}
  & {reword} & 0.34, & 0.14, & 0.153,       & 0.807, & 0.696, & 1.5\\ 
  & {rephrase} & 0.408, & 0.239, & 0.154,    & 0.815, & 0.704, & 1.0\\  
  & {paraphrase} & 0.355, & 0.127, & 0.145, & 0.783, & \textbf{0.721}, & 1.5\\  
  & {expansion} & 0.308, & \textbf{0.098}, & 0.151,  & 0.807, & 0.711, & 1.5\\ 
\midrule
\multirow {5}{*}{Llama-2-13B} & {top-k} & 0.43, & 0.114, & \textbf{0.13},      & \textbf{0.708}, & \textbf{0.72}, & -\\
& {temp-sampling} & 0.43, & 0.099 & 0.121 & 0.702 & 0.733 & - \\
\cdashline{2-8} \rule{0pt}{2.25ex}
 & {reword} & 0.345, & 0.111, & 0.144,    & 0.794, & 0.618, & 1.5\\  
 & {rephrase} & 0.345, & \textbf{0.062}, & 0.141,   & 0.767, & 0.706, & 1.0\\  
 & {paraphrase} & 0.37, & 0.092, & 0.141, & 0.763, & 0.67, & 1.2\\ 
 & {expansion} & 0.36, & 0.138, & 0.138,  & 0.799, & 0.574, & 1.5\\  
\bottomrule
\end{tabular}
\end{table*}

The approximation relies on linearizing the involved functions, however, it is illustrative of the effect of both $s^2_{topk}$ and $s^2_{rephrase}$. In particular, we see that both $s^2_{topk}$ and $s^2_{rephrase}$ act to \emph{temper} the probability $p(A\vert \bz_{mean} , f)$ under the categorical distribution of g. This highlights why using rephrasings with an appropriate temperature might improve the calibration in downstream tasks. In previous works, tempering of the categorical distribution has been found to significantly improve the calibration of deep neural networks \citep{guo2017calibration}. %This also highlights a potential limitation of rephrasings in improving calibration. In particular, we expect that if we temper with an appropriate temperature the results obtained by top-k alone we should be able to match the best performance of rephrasings.

%\section{Experiments on top-k decoding}
In Tables \ref{topk-chal}, \ref{topk-easy} and \ref{topk-opqa} and Figure \ref{fig:topk}, we present the results for the top-k experiments with and without rephrasing, with $k=40$. We also present the relaxed temperature sampling variant \cite{wei2022chain}. We see that the stochasticity of top-40 compared to top-1 decoding from Tables \ref{arc-challenge-main}, \ref{arc-easy-main} and \ref{openbookqa-main} \emph{results in an improvement in calibration but a drop in accuracy. The Brier score often improves at the cost of accuracy}. Further stochasticity in answers caused by rephrasings has a similar effect. These observations are consistent with the fact that top-k and nucleus sampling \citep{holtzman2019curious} make text more human-like but not necessarily more ``accurate". However, if the main goal is calibration, the tables, and Figure \ref{fig:topk} show that in accordance with proposition \ref{theorem_topk_temperature_scaling} rephrasing acts primarily to temper the probability of the top class. This often improves calibration significantly in terms of ECE, and AUROC especially for smaller models. We plot the results of all methods averaged over all models for each dataset in Figure \ref{fig:fullbar}. \emph{Tables \ref{arc-challenge-main}, \ref{arc-easy-main} \ref{openbookqa-main} and \ref{topk-chal}, \ref{topk-easy} \ref{topk-opqa} and Figure \ref{fig:fullbar} indicate that the user should assess whether rephrasing is appropriate after an analysis of his individual model, task and evaluation metric. However, in general, a hyperparameter-optimized choice of rephrasing + top-1 decoding outperforms or matches all other method combinations in all metrics.}

\section{Related works}
The field of estimating the uncertainty of closed-source LLM models is nascent but fast-growing. \cite{kadavath2022language} propose that in addition to a query the user can prompt the LLM to output a numerical confidence value, known as ``verbalized confidence". Crucially, there is no easy statistical justification as to why verbalized confidence should result in calibrated predictions. Uncertainty estimates using verbalized confidence tend to be overly optimistic and concentrate in the 80\%-100\% confidence range \citep{xiong2023can}. %Closer to our approach \cite{} propose to use consistensy based methods for estimating uncertainty. They explore generating multiple responses using only the base query and an above zero LLM temperature which introduces randomness. They also explore concatenating the base query with misleading hints. %We first note that the results of this work have to be treated with caution. \cite{} compare different methods using only the ECE and variants of the AUROC. It is well know that since these metrics are not proper scoring rules, their optimal values can be achieved by classifiers which have no predictive ability but are still "well calibrated". Furthermore, 
%Concatenating a base query with hints requires the memorization of the proposed hint format and is arguably cumbersome to use in practice. By contrast we propose simple rephrasing rules to generate multiple queries which are much easier to use in practice. Finally, in contrast to \cite{}, we also propose a theoretical model that elucidates why using multiple rephrased queries, can result in improved uncertainty estimates. 
Recently, \cite{pacchiardi2023catch} proposed that after submitting a base query the user should ask additional and unrelated binary questions and check the accuracy of the answers. They empirically correlate this to well-calibrated uncertainty but only for the setting where the LLM purposefully lies. Our work is also related to \cite{carlini2024stealing} which manages to ``steal" the last layer of closed-source LLMs using only random queries. 

\citet{wang2022self} proposed to leverage multiple chains of thoughts to derive varied responses. Their findings suggest that a majority vote across these answers not only enhances accuracy but also yields well-calibrated uncertainty estimates. \citet{kuhn2023semantic} introduced a novel uncertainty quantification metric by sampling a multitude of responses and employing a BERT model to categorize these answers. Subsequently, they calculated the entropy of the empirical distribution, presenting an alternative approach to uncertainty estimation. This approach has the significant disadvantage of being computationally expensive and requiring access to a secondary LLM.

Another line of work focuses on rephrasing queries to improve accuracy instead of estimating uncertainty. Specifically, \citet{deng2023rephrase} demonstrated that expanding questions with supplementary details through a zero-shot prompt significantly improves model performance. \citet{zheng2023take} adopted a similar approach by asking the LLM to derive high-level concepts and first principles before reasoning and answering the question, which boosted the performance.

Conversely, another segment of research has delved into the uncertainty estimates derived from the logits associated with multiple-choice questions. This approach entails extracting the logits corresponding to the first token of each option (A, B, C, D) following the question prompt and applying a softmax normalization to ascertain the predicted probabilities for the options. \citet{achiam2023gpt} discovered that while pre-trained models exhibit commendable calibration, the application of Reinforcement Learning from Human Feedback (RLHF) adversely affects calibration. Other studies have endeavored to bolster the calibration of fine-tuned LLMs by employing strategies such as ensembles \citep{wang2023lora} or adopting Bayesian methods \citep{yang2023bayesian}. However, such an approach does not apply to closed-source LLMs where logits are not available, as well as free-form QA tasks.

A comprehensive body of literature exists on the topic of estimating uncertainty in deep neural network models when access to the softmax categorical distribution is available \citep{guo2017calibration, lakshminarayanan2016simple, blundell2015weight, maddox2019simple, wenzel2020good, arbel2023primer}. The most straightforward method involves utilizing the categorical distribution itself as an uncertainty estimate \citep{guo2017calibration}. Noteworthy enhancements can be achieved by applying tempering to the logits just before the application of the softmax function \citep{guo2017calibration}.

%Various advanced methods have been proposed to refine uncertainty estimates, including techniques like deep ensembles \citep{lakshminarayanan2016simple,wang2023lora} and Bayesian approaches, such as the Laplace approximation \citep{ritter2018scalable,yang2023bayesian,yang2024bayesian}, Markov chain Monte Carlo sampling \citep{wenzel2020good}, and Variational Inference \citep{blundell2015weight}. For an extensive overview of these methods, refer to \citet{arbel2023primer}. Fundamentally, these approaches share a common principle of averaging the softmax categorical distribution over multiple minima of the loss, contributing to improved uncertainty estimation in deep neural networks.

\begin{figure*}[t!]
    \centering
    
    \begin{subfigure}{.33\textwidth}
    \includegraphics[width=\textwidth]{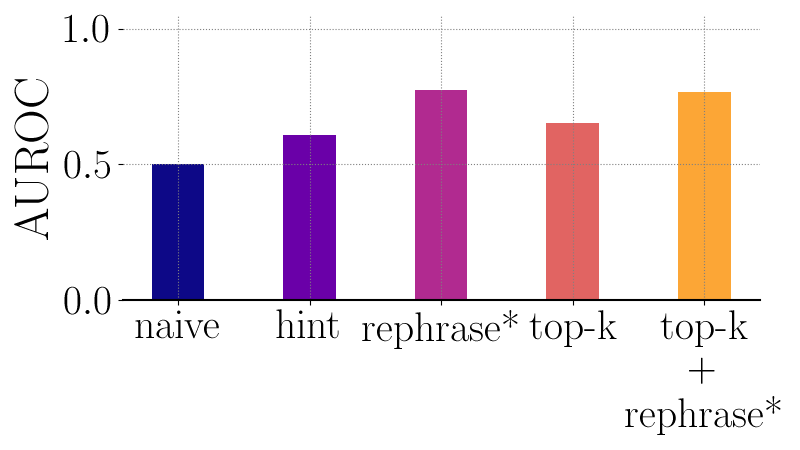}\\
    \caption{ARC-Challenge}
    \label{fig:sub1bar}
    \end{subfigure}%
    %\hspace{15pt}
    \begin{subfigure}{.33\textwidth}
    \includegraphics[width=\textwidth]{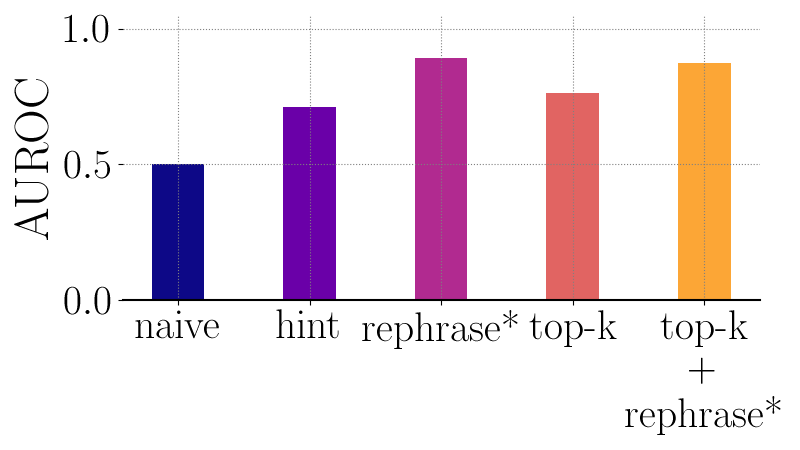}\\
    \caption{ARC-Easy}
    \label{fig:sub2bar}
    \end{subfigure}%
    %\hspace{15pt}
    \begin{subfigure}{.33\textwidth}
    \includegraphics[width=\textwidth]{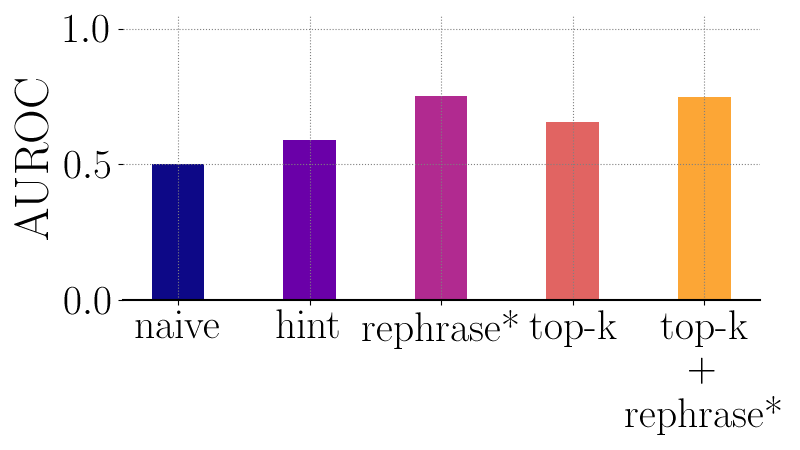}\\
    \caption{OpenBookQA}
    \label{fig:sub3bar}
    \end{subfigure}%
    \caption{We plot the AUROC averaged over all models for each dataset and for each uncertainty estimation method. We observe that top-k improves over the naive top-1 decoding. Furthermore, the best rephrasing method (denoted as rephrase*) improves the AUROC significantly in all cases.}
    \label{fig:fullbar}
\end{figure*}

\section{Discussion}
We conducted a thorough analysis of rephrased queries as a method for obtaining calibrated predictions from closed-source LLM models. Notably, we found that two simple methods; making the query more verbose, and substituting words with their synonyms, provide a straightforward means of identifying false positives. The appeal of our approach lies in its practicality, as it requires only basic language and arithmetic skills by the end user to obtain meaningful uncertainty estimates. Exciting future directions include learning optimal rephrasing rules in a data-driven manner, to be used in conjunction with a rephrasing LLM. While we tested on the multiple choice question setting for ease of evaluation, we expect our results to also hold for open-ended text generation.

\clearpage

\bibliography{colm2024_conference}
\bibliographystyle{colm2024_conference}

\appendix

\section{Prompt template} \label{app:prompt}
We present our prompt template for initiating rephrases with a one-shot example in Table~\ref{table:prompt}. Note that we only present and rephrase questions without revealing choices, to reduce unnecessary bias to rephrases when presented with answer choices.

\begin{table}[H]
\centering
\captionsetup{skip=8pt}
\begin{tabularx}{\textwidth}{c|X}
\hline
\textbf{Method} & \textbf{Prompt} \\
\hline
reword & \textcolor{UserPromptColor}{[INST]Reword the following question:} \\
& \textcolor{UserPromptColor}{George wants to warm his hands quickly by rubbing them. Which skin surface will produce the most heat?} \\
& \textcolor{UserPromptColor}{Respond with the reworded question only: [\textbackslash INST]} \\
& \textcolor{ModelResponseColor}{George seeks to heat his hands swiftly by rubbing them. Which skin area will generate the maximum heat?} \\
& \textcolor{UserPromptColor}{[INST]Reword the following question: \{question\}} \\
& \textcolor{UserPromptColor}{Respond with the reworded question only: [\textbackslash INST]} \\
\hline
rephrase & \textcolor{UserPromptColor}{[INST]Rephrase the following question:} \\
& \textcolor{UserPromptColor}{George wants to warm his hands quickly by rubbing them. Which skin surface will produce the most heat?} \\
& \textcolor{UserPromptColor}{Respond with the rephrased question only: [\textbackslash INST]} \\
& \textcolor{ModelResponseColor}{What type of skin texture on George's hands would generate the most heat through rapid rubbing to warm them effectively?} \\
& \textcolor{UserPromptColor}{[INST]Rephrase the following question: \{question\}} \\
& \textcolor{UserPromptColor}{Respond with the rephrased question only: [\textbackslash INST]} \\
\hline
paraphrase & \textcolor{UserPromptColor}{[INST]Semantically paraphrase the following question:} \\
& \textcolor{UserPromptColor}{George wants to warm his hands quickly by rubbing them. Which skin surface will produce the most heat?} \\
& \textcolor{UserPromptColor}{Respond with the semantically paraphrased question only: [\textbackslash INST]} \\
& \textcolor{ModelResponseColor}{How can George induce the highest thermal output by briskly rubbing his hands, and which part of the skin would be most effective?} \\
& \textcolor{UserPromptColor}{[INST]Semantically paraphrase the following question: \{question\}} \\
& \textcolor{UserPromptColor}{Respond with the semantically paraphrased question only: [\textbackslash INST]} \\
\hline
expansion & \textcolor{UserPromptColor}{[INST] Expand the following question with additional context:} \\
& \textcolor{UserPromptColor}{George wants to warm his hands quickly by rubbing them. Which skin surface will produce the most heat?}  \\
& \textcolor{UserPromptColor}{Respond with the expanded question only: [\textbackslash INST]} \\
& \textcolor{ModelResponseColor}{In the context of seeking immediate relief from the biting cold and understanding the mechanisms behind heat generation through friction, what type of skin texture on George's hands would most effectively generate heat by rapid rubbing?} \\
& \textcolor{UserPromptColor}{[INST]Expand the following question with additional context: \{question\}} \\
& \textcolor{UserPromptColor}{Respond with the expanded question only: [\textbackslash INST]}\\
\hline
\end{tabularx}
\caption{Prompt templates for one-shot rephrasing, with rephrasing methods listed on the left and corresponding prompt on the right. The user instructions are colored in \textcolor{UserPromptColor}{blue} and surrounded by the instruction token, whereas model response demonstrations are colored in \textcolor{ModelResponseColor}{orange}.}
\label{table:prompt}
\end{table}

We followed the instructions in \cite{xiong2023can} to generate ''hint" based rephrasings. Specifically, to generate a rephrased query given a base query, we appended one of the following three weak claims (as they found weak claims outperform other types of hints) together with a random class from the available ones.

\begin{table*}[!t]
    \caption{Different Prompts to Induce Consistency Confidence}
    \label{different_hints}
\centering
\begin{tabular}{cl}
\toprule
 & Consistency Confidence Inducing Prompts  \\
\toprule
Hint1 & I think the answer should be \\ 
Hint2 & I read online the answer is \\
Hint3 & I vaguely remember the answer is \\ 
\bottomrule
\end{tabular}
\end{table*}

\section{Rephrase generations} \label{app:generations}
In the main text, we present rephrases generated by Mistral-7B in Table~\ref{table:generation_mistral}. Here, we present additional generated rephrasings by Llama2-7B and Llama2-13B in Table~\ref{table:generation_llama2_7b} and Table~\ref{table:generation_llama2_13b}.

\begin{table}[t]
\centering
\captionsetup{skip=8pt}
\begin{tabularx}{\textwidth}{c|X}
\hline
\textbf{Method} & \textbf{Question} \\
\hline\hline
original & What part of the digestive system first causes chemical changes to food?
A. Teeth in the mouth. B. Saliva in the mouth. C. Enzymes in the stomach.
D. Enzymes in the small intestine. \\
\hline\hline
reword & What section of the digestive system initiates chemical alterations to sustenance? A. Teeth in the mouth. B. Saliva in the mouth. C. Enzymes in the stomach.
D. Enzymes in the small intestine.\\
\hline
rephrase & Which portion of the digestive system initially catalyzes the biochemical transformation
of ingested sustenance? A. Teeth in the mouth. B. Saliva in the mouth. C. Enzymes in the stomach.
D. Enzymes in the small intestine.\\
\hline
paraphrase &  Which digestive organ releases enzymes that initiate chemical breakdown within ingested
sustenences, leading to nutrient extraction and energy release? A. Teeth in the mouth. B. Saliva in the mouth. C. Enzymes in the stomach.
D. Enzymes in the small intestine.\\
\hline
expansion & In the context of the digestive process and the breakdown of nutrients, which portion
of the digestive system initiates the chemical transformations that result in the
nutrient absorption and energy production? A. Teeth in the mouth. B. Saliva in the mouth. C. Enzymes in the stomach.
D. Enzymes in the small intestine. \\
\hline
\end{tabularx}
\caption{Rephrasing examples generated by Llama2-7B, with rephrasing methods listed on the left and corresponding rephrases on the right.}
\label{table:generation_llama2_7b}
\end{table}

\begin{table}[t]
\centering
\captionsetup{skip=8pt}
\begin{tabularx}{\textwidth}{c|X}
\hline
\textbf{Method} & \textbf{Question} \\
\hline\hline
original & What part of the digestive system first causes chemical changes to food?
A. Teeth in the mouth. B. Saliva in the mouth. C. Enzymes in the stomach.
D. Enzymes in the small intestine. \\
\hline\hline
reword & Which section of the gastrointestinal tract initiates the chemical transformation
of sustenance? A. Teeth in the mouth. B. Saliva in the mouth. C. Enzymes in the stomach.
D. Enzymes in the small intestine.\\
\hline
rephrase & In which section of the digestive system does the initial chemical breakdown of
food occur? A. Teeth in the mouth. B. Saliva in the mouth. C. Enzymes in the stomach.
D. Enzymes in the small intestine.\\
\hline
paraphrase &  In the digestive process, where do crucial transformations initially occur to break
down nutrients? A. Teeth in the mouth. B. Saliva in the mouth. C. Enzymes in the stomach.
D. Enzymes in the small intestine.\\
\hline
expansion & Taking into account that human digestive system's several organs coordinate to breakdown,
absorb, and expel waste, which part of the gastrointestinal system would have the
most significant logic-based influence on the breakdown of food into usable components,
prior to the nutrient absorption? A. Teeth in the mouth. B. Saliva in the mouth. C. Enzymes in the stomach.
D. Enzymes in the small intestine. \\
\hline
\end{tabularx}
\caption{Rephrasing examples generated by Llama2-13B, with rephrasing methods listed on the left and corresponding rephrases on the right.}
\label{table:generation_llama2_13b}
\end{table}

\section{Additional Proofs} \label{app:proof}

\begin{proposition} %\label{theorem_general_noise}
Let $f : \mathcal{X} \rightarrow \mathcal{Y}$ be an LLM, $\bx$ is a base query and $\mathcal{T}(\bx)\sim \tau$ is some randomized transformation of the base query. Let
\begin{equation}
p_A(\bx) = \mathbb{P}\left( f(\mathcal{T}(\bx)) = A\right),
\end{equation}
be the probability of sampling the most probable answer $A \in \mathcal{Y}$ under transformations $\mathcal{T}(\bx)\sim \tau$. Let $\bz_{mean}+\epsilon_{rephrase}$ be the latent representation of $\bx$ under $\mathcal{T}(\bx)$ at the final LLM layer, where $\bz_{mean}$ is the mean representation and $\epsilon_{rephrase}$ is some additive noise. Let $\bw$ be the separating hyperplane between the most probable answer $A$ and the second most probable answer $B$. Assuming that $\bw^{\top} \epsilon_{rephrase}\sim \rho$ follows a logistic distribution with $\mu=0$ and $s=1$ then
\begin{equation}
\begin{split}
    p_A(\bx) = p(A\vert \bz_{mean} , f)
\end{split}
\end{equation}
where $p(A\vert \bz_{mean} , f)$ is the probability of $A$ given $\bz_{mean}$ under the categorical distribution of the final layer.
\end{proposition}
\begin{proof}
We first analyze the categorical distribution, resulting from applying the softmax on the final layer logits. In the binary classification case given a top-1 class prediction $A$, the softmax probability of this class is
\begin{align}
    &p(A\vert \bx , f) = \frac{e^{\bw_{A}^{\top}\bz+b_A}}{e^{\bw_{A}^{\top}\bz+b_A}+e^{\bw_{B}^{\top}\bz+b_B}}\nonumber\\
    &\,\,=\frac{1}{1+e^{-(\bw_{A}+b_A-\bw_{B}-b_B)^{\top}\bz}}=\frac{1}{1+e^{-(\bw^{\top}\bz+b)}}.\label{eq:folk_app}
\end{align}
The above simply corresponds to the folk knowledge that a softmax layer with two classes is equivalent to a single separating hyperplane that assigns classes based on the rule $\sign \left(\bw^{\top}\bz+b\right)$, specifically 
$$g(\bz) = \begin{cases}
    A & \text{if } \left(\bw^{\top}\bz+b\right) > 0, \\
    B & \text{otherwise.}
\end{cases}$$ 
After establishing that the softmax layer is equivalent to this single separating hyperplane, let us relate $p_A(\bx)$ to $\bw^{\top}\bz+b$. We have
\begin{equation}
\begin{split}
p_A(\bx) &=\mathbb{P}\left( f(\mathcal{T}(\bx)) = A \right)\\
&= \mathbb{P}\left( \bw^{\top}(\bz_{mean}+\epsilon_{rephrase})+b > 0 \right)\\
&= \mathbb{P}\left( \bw^{\top}\bz_{mean}+\bw^{\top}\epsilon_{rephrase}+b > 0 \right)\\
&= \mathbb{P}\left( Z>-\bw^{\top}\bz_{mean}-b\right)\\
&= 1-\mathbb{P}\left( Z<-\bw^{\top}\bz_{mean}-b\right)\\
&= 1-F\left(-\bw^{\top}\bz_{mean}-b\right)\\
\end{split}
\end{equation}
Then $F(-\bw^{\top}\bz_{mean}-b) = 1-p_A \iff \bw^{\top}\bz_{mean}+b = -F^{-1}(1-p_A)$. We substitute this result to \ref{eq:folk_app}, assume that $F$ is the cumulative of the logistic distribution with $\mu=0$ and $s=1$ and get
\begin{align}
    p(A\vert \bz_{mean} , f) &= \frac{1}{1+e^{F^{-1}(1-p_A)}}\\
    & = \frac{1}{1+e^{-F^{-1}(p_A(\bx))}}\\
    & = p_A(\bx)
\end{align}
In the second line we used the fact that the inverse cumulative $F^{-1}$ of the logistic distribution is symmetric around $0.5$. In the third line we use the fact that $\frac{1}{1+e^{-x}}$ is the cumulative of the logistic with $\mu=0$ and $s=1$. Thus $p(A\vert \bz_{mean} , f) = F(F^{-1}(p_A(\bx))) \iff p(A\vert \bz_{mean} , f)=p_A(\bx)$

A technical point remains. Even though in the previous we can assume that $g(\bz_{mean})=A$ (that $\bz_{mean}$ results in the most probable class) by definition, we still need to show that $A = \mathrm{argmax}_i \mathbb{P}\left( f(\mathcal{T}(\bx)) = i \right) \iff g(\bz_{mean})=A$. This means that for a closed-source LLM we can identify the (unknown) top-1 class A through Monte Carlo sampling ($A = \mathrm{argmax}_i \mathbb{P}\left( f(\mathcal{T}(\bx)) = i \right)$).

\begin{equation}
\begin{split}
A = \mathrm{argmax}_i \mathbb{P}\left( f(\mathcal{T}(\bx)) = i \right) &\iff \mathbb{P}\left( f(\mathcal{T}(\bx)) = A \right)>\frac{1}{2}\\
&\iff \mathbb{P}\left( \bw^{\top}(\bz_{mean}+\epsilon_{rephrase})+b \geq 0 \right)>\frac{1}{2}\\
&\iff \mathbb{P}\left( \bw^{\top}\bz_{mean}+\bw^{\top}\epsilon_{rephrase}+b \geq 0 \right)>\frac{1}{2}\\
&\iff \mathbb{P}\left( Z\geq-\bw^{\top}\bz_{mean}-b\right)>\frac{1}{2}\\
&\iff \mathbb{P}\left( Z\leq\bw^{\top}\bz_{mean}+b\right)>\frac{1}{2}\\
&\iff \bw^{\top}\bz_{mean}+b > 0\\
&\iff g(\bz_{mean})=A
\end{split}
\end{equation}
where we use the assumption that $Z$ follows a logistic distribution with $\mu=0$ and $s=1$.
\end{proof}

\begin{proposition} %\label{theorem_topk_temperature_scaling}
Let $g : \mathbb{R}^{d_{\eta}} \rightarrow \mathcal{Y}$ be the final encoder layer of an LLM, $\bx$ is a base query and $\mathcal{T}(\bx)\sim \tau$ is some randomized transformation of the base query. Let
\begin{equation}
p_A(\bx) = \mathbb{P}\left( f(\mathcal{T}(\bx)) = A\right),
\end{equation}
be the probability of sampling the most probable answer $f(\bx)=A \in \mathcal{Y}$ under transformations $\mathcal{T}(\bx)\sim \tau$. Let $\bz_{mean}+\epsilon_{topk}+\epsilon_{rephrase}$ be the latent representation of $\bx$ under $\mathcal{T}(\bx)$ at the final LLM layer, where $\bz_{mean}$ is the mean representation and $\epsilon_{topk}$ is additive noise resulting from the top-k decoding and $\epsilon_{rephrase}$ is additive noise resulting from the rephrasings $\mathcal{T}(\bx)$. Assuming that $\bw^{\top} (\epsilon_{topk}+\epsilon_{rephrase})\sim \rho$ approximately follows a logistic distribution with $\mu=0$ and $s=\sqrt{s^2_{topk}+s^2_{rephrase}}$ then
\begin{equation}
\begin{split}
     p_A(\bx) \approx 0.5+\frac{1}{\sqrt{s^2_{topk}+s^2_{rephrase}}}(p(A\vert \bz_{mean} , f)-0.5)
\end{split}
\end{equation}
where $p(A\vert \bz_{mean} , f)$ is the probability of $A$ given $\bz_{mean}$ under the categorical distribution of $g$.
\end{proposition}
\begin{proof}
We first claim that the sum of two logistic distributions $(\mu_1,s_1)$ and $(\mu_1,s_1)$ is approximately logistic with $(\mu_1+\mu_2,\sqrt{s^2_{1}+s^2_{2}})$ by claiming that logistic distributions are approximately Gaussian. Then considering that $p(A\vert \bz_{mean} , f) = \frac{1}{1+e^{F^{-1}(1-p_A(\bx))}}$ we can write
\begin{equation}
\begin{split}
p(A\vert \bz_{mean} , f) &= \frac{1}{1+e^{F^{-1}(1-p_A(\bx))}} = \frac{1}{1+e^{-F^{-1}(p_A(\bx))}} \\
&= 0.5 + \frac{1}{4}F^{-1}(p_A(\bx)) = 0.5 + \frac{1}{4}4\sqrt{s^2_{topk}+s^2_{rephrase}}(p_A(\bx)-0.5) \\
\end{split}
\end{equation}
In the first line we first considered that $F^{-1}$ for the logistic is symmetric thus $F^{-1}(1-p_A(\bx))=-F^{-1}(p_A(\bx))$. In the second line we first do a first order Taylor expansion around $0$ on $\frac{1}{1+e^{-x}}$ and then a first order Taylor expansion around $0.5$ on $F^{-1}$.
\end{proof}

\section{Additional comparisons with CoT} \label{cot_experiments}
We compare with Chain-of-Thought \cite{wei2022chain} for uncertainty estimation and plot the results in Table \ref{CoT_vs_rephrasing-main}. We find that we get competitive results with CoT. At the same time our method is significantly easier and more natural to implement for humans interacting via text with an LLM. In CoT one needs to first obtain a sequence of reasoning steps. These should then be used as additional context when asking an LLM to answer again the base question. By contrast we propose a simple one step process of rephrasing the base question.

\begin{table*}[t]
    \caption{Comparisons between our best rephrasing method and CoT. Our rephrasing method obtains comparable results to CoT in terms of Brier score and other calibration metrics.
    }
    \label{CoT_vs_rephrasing-main}
\centering
\small
\begin{tabular}{cccccccc}
\toprule
Dataset & Model & Method & Acc $\uparrow$ & ECE $\downarrow$ & TACE $\downarrow$ & Brier $\downarrow$  & AUROC $\uparrow$\\
\toprule
\multirow {6}{*}{ARC-C} & \multirow {2}{*}{Mistral-7B} & {CoT}      & 0.725 & 0.173 & 0.071 & 0.439 & 0.719 \\
& & {expansion}                                                        & 0.602 & 0.133 & 0.099 & 0.509 & 0.847 \\ 
\cdashline{2-8} \rule{0pt}{2.25ex}
& \multirow {2}{*}{Llama-2-7B} & {CoT}                              & 0.407 & 0.205 & 0.151 & 0.783 & 0.696 \\
& & {expansion}                                                        & 0.373 & 0.112 & 0.153 & 0.778 & 0.687 \\ 
\cdashline{2-8} \rule{0pt}{2.25ex}
& \multirow {2}{*}{Llama-2-13B} & {CoT}                             & 0.369 & 0.137 & 0.148 & 0.782 & 0.729 \\
& & {reword}                                                           & 0.445 & 0.084 & 0.119 & 0.714 & 0.721 \\ 
\midrule
\multirow {6}{*}{ARC-E} & \multirow {2}{*}{Mistral-7B} & {CoT}      & 0.857 & 0.07  & 0.037 & 0.211 & 0.829 \\
& & {reword}                                                           & 0.753 & 0.045 & 0.062 & 0.297 & 0.931 \\ 
\cdashline{2-8} \rule{0pt}{2.25ex}
& \multirow {2}{*}{Llama-2-7B} & {CoT}                              & 0.482 & 0.104 & 0.116 & 0.624 & 0.842 \\ 
& & {rephrase}                                                         & 0.535 & 0.131 & 0.117 & 0.603 & 0.830 \\ 
\cdashline{2-8} \rule{0pt}{2.25ex}
& \multirow {2}{*}{Llama-2-13B} & {CoT}                             & 0.463 & 0.097 & 0.124 & 0.61  & 0.884 \\
& & {expansion}                                                        & 0.524 & 0.078 & 0.12  & 0.552 & 0.893 \\ 
\midrule
\multirow {6}{*}{OBQA} & \multirow {2}{*}{Mistral-7B} & {CoT}       & 0.662 & 0.153 & 0.083 & 0.501 & 0.762 \\
& & {reword}                                                           & 0.552 & 0.105 & 0.102 & 0.592 & 0.796 \\
\cdashline{2-8} \rule{0pt}{2.25ex}
& \multirow {2}{*}{Llama-2-7B} & {CoT}                              & 0.39  & 0.185 & 0.145 & 0.805 & 0.713 \\
& & {expansion}                                                        & 0.362 & 0.083 & 0.138 & 0.775 & 0.678 \\ 
\cdashline{2-8} \rule{0pt}{2.25ex}
& \multirow {2}{*}{Llama-2-13B} & {CoT}                             & 0.37  & 0.166 & 0.153 & 0.801 & 0.683 \\
& & {rephrase}                                                         & 0.428 & 0.095 & 0.14  & 0.729 & 0.73  \\ 
\bottomrule
\end{tabular}
\end{table*}

\end{document}